\documentclass[preprint,12pt,3p]{elsarticle}
\usepackage{graphicx}
\usepackage{amssymb}
\usepackage{amsthm}
\usepackage{amsmath}
\usepackage{amsfonts}
\usepackage{subfigure}
\usepackage{lineno}
\usepackage{multirow}
\usepackage[colorlinks]{hyperref}
\newtheorem{definition}{Definition}
\newtheorem{theorem}{Theorem}
\newtheorem{example}{Example}
\journal{Neurocomputing}
\begin{document}
\begin{frontmatter}

\title{Block building programming for symbolic regression}

\author[lhd,UCAS]{Chen Chen}
\ead{chenchen@imech.ac.cn}
\author[lhd]{Changtong Luo\corref{cor1}}
\ead{luo@imech.ac.cn}
\author[lhd,UCAS]{Zonglin Jiang}
\ead{zljiang@imech.ac.cn}
 
\cortext[cor1]{Corresponding author}
\address[lhd]{State Key Laboratory of High Temperature Gas Dynamics, Institute of Mechanics, Chinese Academy of Sciences, Beijing 100190, China}
\address[UCAS]{School of Engineering Sciences, University of Chinese Academy of Sciences, Beijing 100049, China}

\begin{abstract}
Symbolic regression that aims to detect underlying data-driven models has become increasingly important for industrial data analysis. For most existing algorithms such as genetic programming (GP), the convergence speed might be too slow for large-scale problems with a large number of variables. This situation may become even worse with increasing problem size. The aforementioned difficulty makes symbolic regression limited in practical applications. Fortunately, in many engineering problems, the independent variables in target models are separable or partially separable. This feature inspires us to develop a new approach, block building programming (BBP). BBP divides the original target function into several blocks, and further into factors. The factors are then modeled by an optimization engine (e.g. GP). Under such circumstances, BBP can make large reductions to the search space. The partition of separability is based on a special method, block and factor detection. Two different optimization engines are applied to test the performance of BBP on a set of symbolic regression problems. Numerical results show that BBP has a good capability of structure and coefficient optimization with high computational efficiency.
\end{abstract}


\end{frontmatter}

\section{Introduction}
\label{Section1}
Data-driven modeling of complex systems has become increasingly important for industrial data analysis when the experimental model structure is unknown or wrong, or the concerned system has changed \cite{Salmeron2017,Yan2016}. Symbolic regression aims to find a data-driven model that can describe a given system based on observed input-response data, and plays an important role in different areas of engineering such as signal processing \cite{Seera2017}, system identification \cite{Ugalde2015}, industrial data analysis \cite{Luo2015-AST}, and industrial design \cite{Parque2017}. Unlike conventional regression methods that require a mathematical model of a given form, symbolic regression is a data-driven approach to extract an appropriate model from a space of all possible expressions $\mathcal{S}$ defined by a set of given binary operations (e.g. $+$, $-$, $\times$, $\div$) and mathematical functions (e.g. $\sin$, $\cos$, $\exp$, $\ln$), which can be described as follows:
\begin{equation}
\label{SRdef}
f^*=\arg\min_{f \in \mathcal{S}}\sum_i\left\|{f(\mathbf{x}^{(i)})-y_i}\right\|,
\end{equation}
where $\mathbf{x}^{(i)}\in{\mathbb{R}^d}$ and $y_i\in{\mathbb{R}}$ are sampling data. $f$ is the target model and $f^*$ is the data-driven model. Symbolic regression is a kind of non-deterministic polynomial (NP) problem, which simultaneously optimizes the structure and coefficient of a target model. How to use an appropriate method to solve a symbolic regression problem is considered as a kaleidoscope in this research field \cite{Chen2017-ICNC,Peng2014,Sotto2017}.

Genetic programming (GP) \cite{Koza1992} is a classical method for symbolic regression. The core idea of GP is to apply Darwin's theory of natural evolution to the artificial world of computers and modeling. Theoretically, GP can obtain accurate results, provided that the computation time is long enough. However, describing a large-scale target model with a large number of variables is still a challenging task. This situation may become even worse with increasing problem size (increasing number of independent variables and range of these variables). This is because the target model with a large number of variables may result in large search depth and high computational costs of GP. The convergence speed of GP may then be too slow. This makes GP very inconvenient in engineering applications.

Apart from basic GP, two groups of methods for symbolic regression have been studied. The first group focused on evolutionary strategy, such as grammatical evolution \cite{ONeil2003} and parse-matrix evolution \cite{Luo2012-PME}. These variants of GP can simplify the coding process. Gan et al. \cite{Gan2009} introduced a clone selection programming method based on an artificial immune system. Karaboga et al. \cite{Karaboga2012} proposed an artificial bee colony programming method based on the foraging behavior of honeybees. However, these methods are still based on the idea of biological simulation processes. This helps little to improve the convergence speed when solving large-scale problems.

The second branch exploited strategies to reduce the search space of the solution. McConaghy \cite{McConaghy2011} presented the first non-evolutionary algorithm, fast function eXtraction (FFX), based on pathwise regularized learning, which confined its search space to generalized linear space. However, the computational efficiency is gained at the sacrifice of losing the generality of the solution. More recently, Worm \cite{Worm2016} proposed a deterministic machine-learning algorithm, prioritized grammar enumeration (PGE). PGE merges isomorphic chromosome presentations (equations) into a canonical form. The author argues that it could make a large reduction to the search space. However, debate still remains on how the simplification affects the solving process \cite{Kinzett2008,Kinzett2009,McRee2010}.

In many scientific or engineering problems, the target models are separable. Luo et al. \cite{Luo2017} presented a divide-and-conquer (D\&C) method for GP. The authors indicated that detecting the correlation between each variable and the target function could accelerate the solving process. D\&C can decompose a concerned separable model into a number of sub-models, and then optimize them. The separability is probed by a special method, the bi-correlation test (BiCT). However, the D\&C method can only be valid for an additively/multiplicatively separable target model (see Definition \ref{def1-add-mul} in Section \ref{Section2}). Many practical models are out of the scope of the separable model (Eq. (\ref{Separpluseq}) and (\ref{Separtimeseq})). This limits the D\&C method for further applications.

In this paper, a more general separable model that may involve mixed binary operators, namely plus ($+$), minus ($-$), times ($\times$), and division ($\div$), is introduced. In order to get the structure of the generalized separable model, a new approach, block building programming (BBP), for symbolic regression is also proposed. BBP reveals the target separable model using a block and factor detection process, which divides the original model into a number of blocks, and further into factors. Meanwhile, binary operators could also be determined. The method can be considered as a bi-level D\&C method. The separability is detected by a generalized BiCT method. Numerical results show that BBP can obtain the target functions more reliably, and produce extremely large accelerations of the GP method for symbolic regression.

The presentation of this paper is organized as follows. Section \ref{Section2} is devoted to the more general separable model. The principle and procedure of the BPP approach are described in Section \ref{Section3}. Section \ref{Section4} presents numerical results, discussions, and efficiency analysis for the proposed method. In the last section, conclusions are drawn with future works.

\section{Definition of separability}
\label{Section2}

\subsection{Examples}
As previously mentioned, in many applications, the target models are separable. Below, two real-world problems are given to illustrate separability.

\begin{example}
\label{ex2-1}
When developing a rocket engine, it is crucial to model the internal flow of a high-speed compressible gas through the nozzle. The closed-form expression for the mass flow through a choked nozzle \cite{Anderson2006} is
\begin{equation}
\label{nozzle}
\dot m = \frac{{{p_0}{A^*}}}{{\sqrt {{T_0}} }}\sqrt {\frac{\gamma }{R}{{\left( {\frac{2}{{\gamma  + 1}}} \right)}^{{{\left( {\gamma  + 1} \right)} \mathord{\left/
 {\vphantom {{\left( {\gamma  + 1} \right)} {\left( {\gamma  - 1} \right)}}} \right.
 \kern-\nulldelimiterspace} {\left( {\gamma  - 1} \right)}}}}},
\end{equation}
where $p_0$ and $T_0$ represent the total pressure and total temperature, respectively. $A^*$ is the sonic throat area. $R$ is the specific gas constant, which is a different value for different gases. $\gamma  = {c_p}/{c_v}$, where ${c_v}$ and ${c_p}$ are the specific heat at constant volume and constant pressure. The sub-functions of the five independent variables, $p_0$, $T_0$, $A^*$, $R$, and $\gamma$ are all multiplicatively separable in Eq. (\ref{nozzle}). That is, the target model can be re-expressed as follows
\begin{equation}
\label{nozzle2}
\begin{aligned}
\dot m & = f\left( {{p_0},{A^*},{T_0},R,\gamma } \right) \\
& = {f _1}\left( {{p_0}} \right) \times {f _2}\left( {{A^*}} \right) \times {f _3}\left( {{T_0}} \right) \times {f _4}\left( R \right) \times {f _5}\left( \gamma  \right).
\end{aligned}
\end{equation}
\end{example}

The target function with five independent variables can be divided into five sub-functions that are multiplied together, and each with only one independent variable. Furthermore, the binary operator between two sub-functions could be plus ($+$) or times ($\times$). 

\begin{example}
\label{ex2-2}
In aircraft design, the lift coefficient of an entire airplane \cite{Raymer2012} can be expressed as
\begin{equation}
\label{CL}
{C_L} = {C_{L\alpha }}\left( {\alpha  - {\alpha_0}} \right) + {C_{L{\delta _e}}}{\delta _e}\frac{{{S_{{\text{HT}}}}}}{{{S_{{\text{ref}}}}}},
\end{equation}
where $C_{L\alpha}$ and $C_{L{\delta _e}}$ are the lift slope of the body wings and tail wings. $\alpha$, $\alpha_0$, and $\delta _e$ are the angle of attack, zero-lift angle of attack, and deflection angle of the tail wing, respectively. $S_{\text{HT}}$ and $S_{\text{ref}}$ are the tail wing area and reference area, respectively. Note that the sub-functions of the variable $C_{L\alpha }$, $C_{L{\delta _e}}$, $\delta _e$, $S_{{\text{HT}}}$, and $S_{{\text{ref}}}$ are separable, but not purely additively/multiplicatively separable. Variables $\alpha$ and ${\alpha_0}$ are not separable, but their combination $\left( \alpha,\alpha_0 \right)$ can be considered separable. Hence, Eq. (\ref{CL}) can be re-expressed as
\begin{equation}
\label{CL2}
\begin{aligned}
{C_L} & =  f\left( {{C_{L\alpha }},\alpha ,{\alpha _0},{C_{L{\delta _e}}},{\delta _e},{S_{{\text{HT}}}},{S_{{\text{ref}}}}} \right) \\
& = {f _1}\left( {{C_{L\alpha }}} \right) \times {f _2}\left( {\alpha ,{\alpha _0}} \right) + {f _3}\left( {{C_{L{\delta _e}}}} \right) \times {f _4}\left( {{\delta _e}} \right)\times{f _5}\left( {{S_{{\text{HT}}}}} \right) \times {f _6}\left( {{S_{{\text{ref}}}}} \right).
\end{aligned}
\end{equation}
In this example, the target function is divided into six sub-functions.
\end{example}

\subsection{Additively/multiplicatively separable model}
The additively and multiplicatively separable models introduced in \cite{Luo2017} are briefly reviewed below.
\begin{definition}
\label{def1-add-mul}
A scalar function $f\left( {\mathbf{x}} \right)$ with $n$ continuous variables ${\mathbf{x}} = {\left[ {{x_1},{x_2}, \cdots ,{x_n}} \right]^\top}$ ($f:{\mathbb{R}^n} \mapsto \mathbb{R}$, $\mathbf{x} \in {\mathbb{R}^n}$) is additively separable if and only if it can be rewritten as
\begin{equation}
\label{Separpluseq}
f\left( {\mathbf{x}} \right) = {\alpha_0} + \sum\limits_{i = 1}^m {{\alpha_i}{f_i}\left( {{{\mathbf{I}}^{\left( i \right)}}{\mathbf{x}}} \right)} ,
\end{equation}
and is multiplicatively separable if and only if it can be rewritten as
\begin{equation}
\label{Separtimeseq}
f\left( {\mathbf{x}} \right) = {\alpha_0} \cdot \prod\limits_{i = 1}^m {{f_i}\left( {{{\mathbf{I}}^{\left( i \right)}}{\mathbf{x}}} \right)}.
\end{equation}
In Eq. (\ref{Separpluseq}) and (\ref{Separtimeseq}), ${{\mathbf{I}}^{\left( i \right)}} \in {\mathbb{R}^{{n_i} \times n}}$ is the partitioned matrix of the identity matrix ${\mathbf{I}} \in {\mathbb{R}^{n \times n}}$, namely ${\mathbf{I}} = {\left[ {\begin{array}{*{20}{c}}
  {{{\mathbf{I}}^{\left( 1 \right)}}}&{{{\mathbf{I}}^{\left( 2 \right)}}}& \cdots &{{{\mathbf{I}}^{\left( m \right)}}} 
\end{array}} \right]^\top}$, $\sum\nolimits_{i = 1}^m {{n_i}}  = n$. ${{\mathbf{I}}^{\left( i \right)}}{\mathbf{x}}$ is the variables set with $n_i$ elements. $n_i$ represents the number of variables in sub-function $f_i$. Sub-function $f_i$ is a scalar function such that $f_i:{\mathbb{R}^{{n_i}}} \mapsto \mathbb{R}$. $\alpha_i$ is a constant coefficient.
\end{definition}

\subsection{Partially/completely separable model}
\label{Section2.1}
Based on the definition of additive/multiplicative separability, the new separable model with mixed binary operators, namely plus ($+$), minus ($-$), times ($\times$), and division ($\div$) are defined as follows.

\begin{definition}
\label{def1}
A scalar function $f\left( {\mathbf{x}} \right)$ with $n$ continuous variables ${\mathbf{x}} = {\left[ {{x_1},{x_2}, \cdots ,{x_n}} \right]^\top}$ ($f:{\mathbb{R}^n} \mapsto \mathbb{R}$, $\mathbf{x} \in {\mathbb{R}^n}$) is partially separable if and only if it can be rewritten as
\begin{equation}
\label{SeparFuncEqu}
f\left( {\mathbf{x}} \right) = {\alpha_0}{ \otimes _1}{\alpha_1}{f_1}\left( {{{\mathbf{I}}^{\left( 1 \right)}}{\mathbf{x}}} \right){ \otimes _2}{\alpha_2}{f_2}\left( {{{\mathbf{I}}^{\left( 2 \right)}}{\mathbf{x}}} \right){ \otimes _3} \cdots { \otimes _m}{\alpha_m}{f_m}\left( {{{\mathbf{I}}^{\left( m \right)}}{\mathbf{x}}} \right),
\end{equation}
where ${{\mathbf{I}}^{\left( i \right)}} \in {\mathbb{R}^{{n_i} \times n}}$ is the partitioned matrix of the identity matrix ${\mathbf{I}} \in {\mathbb{R}^{n \times n}}$, namely ${\mathbf{I}} = {\left[ {\begin{array}{*{20}{c}}
  {{{\mathbf{I}}^{\left( 1 \right)}}}&{{{\mathbf{I}}^{\left( 2 \right)}}}& \cdots &{{{\mathbf{I}}^{\left( m \right)}}} 
\end{array}} \right]^\top}$, $\sum\nolimits_{i = 1}^m {{n_i}}  = n$. ${{\mathbf{I}}^{\left( i \right)}}{\mathbf{x}}$ is the variables set with $n_i$ elements. $n_i$ represents the number of variables in sub-function $f_i$. Sub-function $f_i$ is a scalar function such that $f_i:{\mathbb{R}^{{n_i}}} \mapsto \mathbb{R}$. The binary operator $\otimes_i$ can be plus ($+$) and times ($\times$). $\alpha_i$ is a constant coefficient.
\end{definition}

Note that the binary operators minus ($-$) and division ($/$) are not included in $\otimes$ for simplicity. This does not affect much of its generality, since minus ($-$) could be regarded as $\left(  -  \right) = \left( { - 1} \right) \cdot \left(  +  \right)$, and sub-function could be treated as ${\tilde f_i}\left(  \cdot  \right) = 1/{f_i}\left(  \cdot  \right)$ if only ${f _i}\left(  \cdot  \right) \ne 0$.
 
\begin{definition}
\label{def2}
A scalar function $f\left( {\mathbf{x}} \right)$ with $n$ continuous variables ($f:{\mathbb{R}^n} \mapsto \mathbb{R}$, $\mathbf{x} \in {\mathbb{R}^n}$) is completely separable if and only if it can be rewritten as Eq. (\ref{SeparFuncEqu}) and $n_i=1$ for all $i=1, 2, \cdots, m$.
\end{definition}

\section{Block building programming}
\label{Section3}

\subsection{Bi-correlation test}
\label{Section3.1}
The bi-correlation test (BiCT) method proposed in \cite{Luo2017} is used to detect whether a concerned target model is additively or multiplicatively separable. BiCT is based on random sampling and the linear correlation method.

\subsection{Block and factor detection}
\label{Section3.2}
The additively or multiplicatively separable target function can be easily detected by the BiCT. However, how to determine each binary operator $\otimes _i$ of Eq. (\ref{SeparFuncEqu}) is a critical step in BBP. One way is to recognize each binary operator $\otimes _i$ sequentially with random sampling and linear correlation techniques. For example, a given target function of six variables with five sub-functions is given below
\begin{equation}
\label{example3-1}
\begin{aligned}
  f\left( {{x_1}, \cdots ,{x_6}} \right) &  = {\alpha _0}{ \otimes _1}{\alpha _1}{f_1}\left( {{x_1}} \right){ \otimes _2}{\alpha _2}{f_2}\left( {{x_2},{x_3}} \right){ \otimes _3}{\alpha _3}{f_3}\left( {{x_4}} \right){ \otimes _4}{\alpha _4}{f_4}\left( {{x_5}} \right){ \otimes _5}{\alpha _5}{f_5}\left( {{x_6}} \right) \\ 
   &  = {\alpha _0} + {\alpha _1}{f_1}\left( {{x_1}} \right) \times {\alpha _2}{f_2}\left( {{x_2},{x_3}} \right) + {\alpha _3}{f_3}\left( {{x_4}} \right) + {\alpha _4}{f_4}\left( {{x_5}} \right) \times {\alpha _5}{f_5}\left( {{x_6}} \right). \\ 
\end{aligned}
\end{equation}
The first step is to determine the binary operator $\otimes_1$. The six variables are sampled with the variable $x_1$ changed, and the remaining variables $x_2, x_3, \cdots, x_6$ fixed. However, it is found that the variable $x_1$ cannot be separable from the variables $x_2,x_3,\cdots,x_6$, since the operation order of the two binary operators plus ($+$) and times ($\times$) is different. This indicates that recognizing each $\otimes _i$ sequentially is difficult.

To overcome the aforementioned difficulty, block and factor detection is introduced, which helps to recognize the binary operator $\otimes _i$ more effectively. Before introducing this method, a theorem is given as follows.

\begin{theorem}
\label{theorem3}
Eq. (\ref{SeparFuncEqu}) can be equivalently written as
\begin{equation}
\label{SeparFuncEquDetail}
f\left( {\mathbf{x}} \right) = {\beta_0} + \sum\limits_{i = 1}^p {{\beta_i}{\varphi _i}\left( {{{\mathbf{I}}^{\left( i \right)}}{\mathbf{x}}} \right)}  = {\beta_0} + \sum\limits_{i = 1}^p {{\beta_i}\prod\limits_{j = 1}^{{q_i}} {{\psi _{i,j}}\left( {{\mathbf{I}}_j^{\left( i \right)}{\mathbf{x}}} \right)} } ,
\end{equation}
where ${\mathbf{x}} = {\left[ {{x_1},{x_2}, \cdots ,{x_n}} \right]^\top} \in {\mathbb{R}^n}$ and $f:{\mathbb{R}^n} \mapsto \mathbb{R}$. ${{\mathbf{I}}^{\left( i \right)}} \in {\mathbb{R}^{{s_i} \times n}}$ is the partitioned matrix of the identity matrix ${\mathbf{I}} \in {\mathbb{R}^{n \times n}}$, namely ${\mathbf{I}} = {\left[ {\begin{array}{*{20}{c}}
  {{{\mathbf{I}}^{\left( 1 \right)}}}&{{{\mathbf{I}}^{\left( 2 \right)}}}& \cdots &{{{\mathbf{I}}^{\left( p \right)}}} 
\end{array}} \right]^\top}$, $\sum\nolimits_{i = 1}^p {{s_i}}  = n$. ${{{\mathbf{I}}_{j}^{\left( i \right)}}\in {\mathbb{R}^{{s_{i,j}} \times n}}}$ is the partitioned matrix of the ${{\mathbf{I}}^{\left( i \right)}}$, namely ${{\mathbf{I}}^{\left( i \right)}} = \left[ {\begin{array}{*{20}{c}}
  {{\mathbf{I}}_1^{\left( i \right)}}&{{\mathbf{I}}_2^{\left( i \right)}}& \cdots &{{\mathbf{I}}_{{q_i}}^{\left( i \right)}} 
\end{array}} \right]^\top$, $\sum\nolimits_{j = 1}^{q_i} {{s_{i,j}}}  = s_i$. $p \geqslant 1$, ${q_i} \geqslant 1$, $\sum\nolimits_{i = 1}^p {{q_i}}  = m$. Sub-functions ${{\varphi _i}}$ and ${\psi _{i,j}}$ are scalar functions such that ${\varphi _i}:{\mathbb{R}^{{s_{i}}}} \mapsto \mathbb{R}$ and ${\psi _{i,j}}:{\mathbb{R}^{{s_{i,j}}}} \mapsto \mathbb{R}$.

\end{theorem}

\begin{proof}
See \ref{appendixA}.
\end{proof}

\begin{definition}
\label{def3-block_factor}
The sub-function ${{\varphi _i}\left( {{{\mathbf{I}}^{\left( i \right)}}{\mathbf{x}}} \right)}$ is the $i$-th block of Eq. (\ref{SeparFuncEquDetail}), and the sub-function ${{\psi _{i,j}}\left( {{\mathbf{I}}_j^{\left( i \right)}{\mathbf{x}}} \right)}$ is the $j$-th factor of the $i$-th block.
\end{definition}

It is observed from Eq. (\ref{example3-1}) that there are three blocks, ${\varphi _1}\left( {{x_1},{x_2},{x_3}} \right)$, ${\varphi _2}\left( {{x_4}} \right)$, and ${\varphi _3}\left( {{x_5},{x_6}} \right)$. The structure of Eq. (\ref{example3-1}) is given as the following equation, based on the Theorem \ref{theorem3},

\begin{equation}
\label{SeparFuncEquDetail-2}
\begin{aligned}
  f\left( {{x_1}, \cdots ,{x_6}} \right) &  = {\alpha _0} + {\alpha _1}{f_1}\left( {{x_1}} \right) \times {\alpha _2}{f_2}\left( {{x_2},{x_3}} \right) + {\alpha _3}{f_3}\left( {{x_4}} \right) + {\alpha _4}{f_4}\left( {{x_5}} \right) \times {\alpha _5}{f_5}\left( {{x_6}} \right) \\ 
   &  = {\beta _0} + {\beta _1}{\varphi _1}\left( {{x_1},{x_2},{x_3}} \right) + {\beta _2}{\varphi _2}\left( {{x_4}} \right) + {\beta _3}{\varphi _3}\left( {{x_5},{x_6}} \right) \\ 
   &  = {\beta _0} + {\beta _1}\boxed{{\psi _{1,1}}\left( {{x_1}} \right) \times {\psi _{1,2}}\left( {{x_2},{x_3}} \right)} + {\beta _2}\boxed{{\psi _{2,1}}\left( {{x_4}} \right)} + {\beta _3}\boxed{{\psi _{3,1}}\left( {{x_5}} \right) \times {\psi _{3,2}}\left( {{x_6}} \right)}. \\ 
\end{aligned}
\end{equation}

The first block has two factors, ${\psi _{1,1}}\left( {{x_1}} \right)$ and ${\psi _{1,2}}\left( {{x_2},{x_3}} \right)$. The second block has only one factor, ${\psi _{2,1}}\left( {{x_4}} \right)$. The last block also has two factors, ${\psi _{3,1}}\left( {{x_5}} \right)$ and ${\psi _{3,2}}\left( {{x_6}} \right)$. Note that variables $x_2$ and $x_3$ are partially separable, while variables $x_1$, $x_4$, $x_5$, and $x_6$ are all completely separable.

From the previous discussion, it is straightforward to show that in Eq. (\ref{SeparFuncEquDetail}), the original target function $f\left( {\mathbf{x}} \right)$ is first divided into several blocks ${\varphi _i}\left(  \cdot  \right)$ with global constants $c_i$. Meanwhile, all binary plus ($+$) operators are determined, which is based on the separability detection of the additively separable ${\varphi _i}\left(  \cdot  \right)$ by BiCT, where $i=1,2,\cdots,p$. Next, each block ${\varphi _i}\left(  \cdot  \right)$ is divided into several factors ${\psi _{i,j}}\left(  \cdot  \right)$. Meanwhile, all binary times ($\times$) operators are determined, which is based on the separability detection of multiplicatively separable ${\psi _{i,j}}\left(  \cdot  \right)$ by BiCT method, where $j=1,2,\cdots,q_i$. It is clear that the process of block and factor detection does not require any special optimization engine.

\subsection{Factor modeling}
\label{Section3.3}
The mission of symbolic regression is to optimize both the structure and coefficient of a target function that describes an input-response system. In block building programming (BBP), after the binary operators are determined, the original target function $f\left( {\mathbf{x}} \right)$ is divided into several factors ${\psi _{i,j}}\left(  \cdot  \right)$. In this section, we aim to find a proper way to model these factors. This problem is quite easy to be solved by an optimization algorithm, since the structure and coefficient of a factor ${\psi _{i,j}}\left(  \cdot  \right)$ can be optimized while the rest are kept fixed and unchanged.

Without the loss of generality, the factor ${\psi _{1,1}}\left( {{x_1},{x_2}, \cdots ,{x_{s_{1,1}}}} \right)$ in Eq. (\ref{SeparFuncEquDetail}) illustrates the implementation of the modeling process.

\begin{enumerate}[1.]
\item Let the matrix $\mathbf{X}$ be a set of $N$ sampling points for all $n$ independent variables,
\begin{equation}
\label{X}
{\mathbf{X}} = \left[ {\begin{array}{*{20}{c}}
  {{x_{11}}}&{{x_{12}}}& \cdots &{{x_{1,n}}} \\ 
  {x{}_{21}}&{{x_{22}}}& \cdots &{{x_{2,n}}} \\ 
   \vdots & \vdots &{}& \vdots  \\ 
  {{x_{N,1}}}&{{x_{N,2}}}& \cdots &{{x_{N,n}}} 
\end{array}} \right],
\end{equation}
where ${x_{i,j}}$ represents the $i$-th sampling point of the $j$-th independent variable, $i = 1,2, \cdots ,n;j = 1,2, \cdots ,N$.
\item Keep variables ${{x_1},{x_2}, \cdots ,{x_{s_{1,1}}}}$ being randomly sampled. Let the sampling points of the variables in local block (block 1), ${x_{{s_{1,1}} + 1}},{x_{{s_{1,1}} + 2}}, \cdots ,{x_{{s_1}}}$, be fixed to any two given points $x_A$ and $x_B$ $\left( {\forall {x_A},{x_B} \in \left[ {a,b} \right]} \right)$, respectively. In addition, let the sampling points of variables in other blocks (blocks 2 to $p$), namely ${x_{{s_1} + 1}},{x_{{s_1} + 2}}, \cdots ,{x_n}$, be fixed to a given point $x_G$ ($\forall x_G \in \left[ {a,b} \right]$). We obtain
\begin{equation}
\label{X1}
{{\mathbf{X}}_1} = \left[ {\begin{array}{*{20}{c}}
  {{x_{1,1}}}& \cdots &{{x_{1,{s_{1,1}}}}}&{x_{1,{s_{1,1}} + 1}^{\left( A \right)}}& \cdots &{x_{1,{s_1}}^{\left( A \right)}}&{x_{1,{s_1} + 1}^{\left( G \right)}}& \cdots &{x_{1,n}^{\left( G \right)}} \\ 
  {{x_{2,2}}}& \cdots &{{x_{2,{s_{1,1}}}}}&{x_{2,{s_{1,1}} + 1}^{\left( A \right)}}& \cdots &{x_{2,{s_1}}^{\left( A \right)}}&{x_{2,{s_1} + 1}^{\left( G \right)}}& \cdots &{x_{2,n}^{\left( G \right)}} \\ 
   \vdots &{}& \vdots & \vdots &{}& \vdots & \vdots &{}& \vdots  \\ 
  {{x_{N,1}}}& \cdots &{{x_{N,{s_{1,1}}}}}&{x_{N,{s_{1,1}} + 1}^{\left( A \right)}}& \cdots &{x_{N,{s_1}}^{\left( A \right)}}&{x_{N,{s_1} + 1}^{\left( G \right)}}& \cdots &{x_{2,n}^{\left( G \right)}} 
\end{array}} \right],
\end{equation}
and
\begin{equation}
\label{X2}
{{\mathbf{X}}_2} = \left[ {\begin{array}{*{20}{c}}
  {{x_{1,1}}}& \cdots &{{x_{1,{s_{1,1}}}}}&{x_{1,{s_{1,1}} + 1}^{\left( B \right)}}& \cdots &{x_{1,{s_1}}^{\left( B \right)}}&{x_{1,{s_1} + 1}^{\left( G \right)}}& \cdots &{x_{1,n}^{\left( G \right)}} \\ 
  {{x_{2,2}}}& \cdots &{{x_{2,{s_{1,1}}}}}&{x_{2,{s_{1,1}} + 1}^{\left( B \right)}}& \cdots &{x_{2,{s_1}}^{\left( B \right)}}&{x_{2,{s_1} + 1}^{\left( G \right)}}& \cdots &{x_{2,n}^{\left( G \right)}} \\ 
   \vdots &{}& \vdots & \vdots &{}& \vdots & \vdots &{}& \vdots  \\ 
  {{x_{N,1}}}& \cdots &{{x_{N,{s_{1,1}}}}}&{x_{N,{s_{1,1}} + 1}^{\left( B \right)}}& \cdots &{x_{N,{s_1}}^{\left( B \right)}}&{x_{N,{s_1} + 1}^{\left( G \right)}}& \cdots &{x_{2,n}^{\left( G \right)}} 
\end{array}} \right].
\end{equation}
\item Let ${\mathbf{\tilde X}} = {{\mathbf{X}}_1} - {{\mathbf{X}}_2} = \left[ {\begin{array}{*{20}{c}}
  {{{\mathbf{X}}_{{\text{train}}}}}&{\mathbf{0}} 
\end{array}} \right]$. Matrix ${{{\mathbf{X}}_{{\text{train}}}}}$ is a partition of matrix ${\mathbf{\tilde X}}$. Next, let ${{\mathbf{f}}_{\text{train}}}$ be the vector of which the $i$-th element is the function value of the $i$-th row of matrix ${\mathbf{\tilde X}}$, namely ${{\mathbf{f}}_{{\text{train}}}} = f\left( {\mathbf{\tilde X}} \right)$.
\item Substitute ${{\mathbf{f}}_{{\text{train}}}}$ and ${{{\mathbf{X}}_{{\text{train}}}}}$ into the fit models $y_{\text{train}} = \beta \cdot f^*\left( x_{\text{train}} \right)$. This step could be realized by an existing optimization engine (e.g. GP). Note that, the constant $\beta$ represent the fitting parameter of the function of variables ${x_{{s_{1,1}} + 1}},{x_{{s_{1,1}} + 2}}, \cdots ,{x_{{s_1}}}$, since these variables are unchanged during this process. We aim to obtain the optimization model $f^*$, and constant $\beta$ will be discarded.
\end{enumerate}

Other factors ${\psi _{i,j}}$ could be obtained in the same way. In fact, many state-of-the-art optimization engines are valid for BBP. Genetic programming methods (e.g. parse-matrix evolution (PME) \cite{Luo2012-PME} and GPTIPS \cite{Searson2010}), swarm intelligence methods (e.g. artificial bee colony programming (ABCP) \cite{Karaboga2012}), and global optimization methods (e.g. low-dimensional simplex evolution (LDSE) \cite{Luo2012-LDSE}) are all easy to power BBP.

\subsection{Block building programming}
\label{Section3.5}
Block-building programming (BBP) can be considered as a bi-level D\&C method, that is, the separability detection involves two processes (block and factor detection). In fact, BBP provides a framework of genetic programming methods or global optimization algorithms for symbolic regression. The main process of BBP is decomposed and presented in previous sections (Section \ref{Section3.1} to \ref{Section3.3}). Despite different optimization engines for factors modeling being used, the general procedure of BBP could be described as follows.

\begin{description}
\item {Procedure of BBP:}
\item {Step 1.} (Initialization) Input the dimension of the target function $D$, the set $S = \left\{ {i:i = 1,2, \cdots ,D} \right\}$ for initial variables subscript number, sampling interval $\left[ {a,b} \right]$, and the number of sampling points $N$. Generate a sampling set $\mathbf{X} \in \left[ {a,b} \right] \subset {\mathbb{R}^{{N} \times D}}$.
\item {Step 2.} (Block detection) The information (the subscript number of the local block and variables) of each block ${\varphi _i \left(  \cdot  \right)}$, $i = 1,2, \cdots ,p$, is detected and preserved iteratively by BiCT (that is, the additively separable block).
\item {Step 3.} (Factor detection) For each block ${\varphi _i \left(  \cdot  \right)}$, the information (the subscript number of the local block, factor, and variables) of each factor ${\psi _{i,j} \left(  \cdot  \right)}$, $j = 1,2, \cdots ,q_{i}$, is detected and preserved iteratively by BiCT (that is, the multiplicatively separable factor in local block).
\item {Step 4.} (Factor modeling) For the $j$-th factor ${\psi _{i,j} \left(  \cdot  \right)}$ in the $i$-th block, set the variables in sampling set $\mathbf{X}$ in blocks $\left\{ {1,2,\cdots,i-1,i+1,\cdots,p} \right\}$ to be fixed to $x_G$, and the variables in factors $\left\{ {1,2,\cdots,j-1,j+1,\cdots,q_{j}} \right\}$ of the $i$-th block to be fixed to $x_A$ and $x_B$. Let ${{\mathbf{\tilde X}}^{i,j}} = {\mathbf{X}}_1^{i,j} - {\mathbf{X}}_2^{i,j} = \left[ {\begin{array}{*{20}{c}}
  {{\mathbf{X}}_{{\text{train}}}^{i,j}}&{\mathbf{0}} 
\end{array}} \right]$ and ${\mathbf{f}}_{{\text{train}}}^{i,j} = f\left( {{{{\mathbf{\tilde X}}}^{i,j}}} \right)$. The optimization engine is then used.
\item {Step 5.} (Global assembling) Global parameter $\beta_k$, $k = 0,1, \cdots ,p$, can be linearly fitted by the equation ${{\mathbf{f}}_{{\text{train}}}} = {\beta_0} + \sum\limits_{i = 1}^p {{\beta_i}{\varphi _i}\left(  \cdot  \right)}  = {\beta_0} + \sum\limits_{i = 1}^p {{\beta_i}\prod\limits_{j = 1}^{{q_i}} {{\psi _{i,j}}\left( {{{\mathbf{X}}_{{\text{train}}}}} \right)} }$.
\end{description}

It is clear from the above procedure that the optimization process of BBP could be divided into two parts, inner and outer optimization. The inner optimization (e.g. LDSE and GPTIPS) is invoked to optimize the structure and coefficients of each factor, with the variables of other factors being fixed. The outer optimization aims to optimize the global parameters of the target model structure. An example procedure of BBP is provided in Fig. \ref{fig1}.

\begin{figure}
\centering
\includegraphics[width=0.8\linewidth]{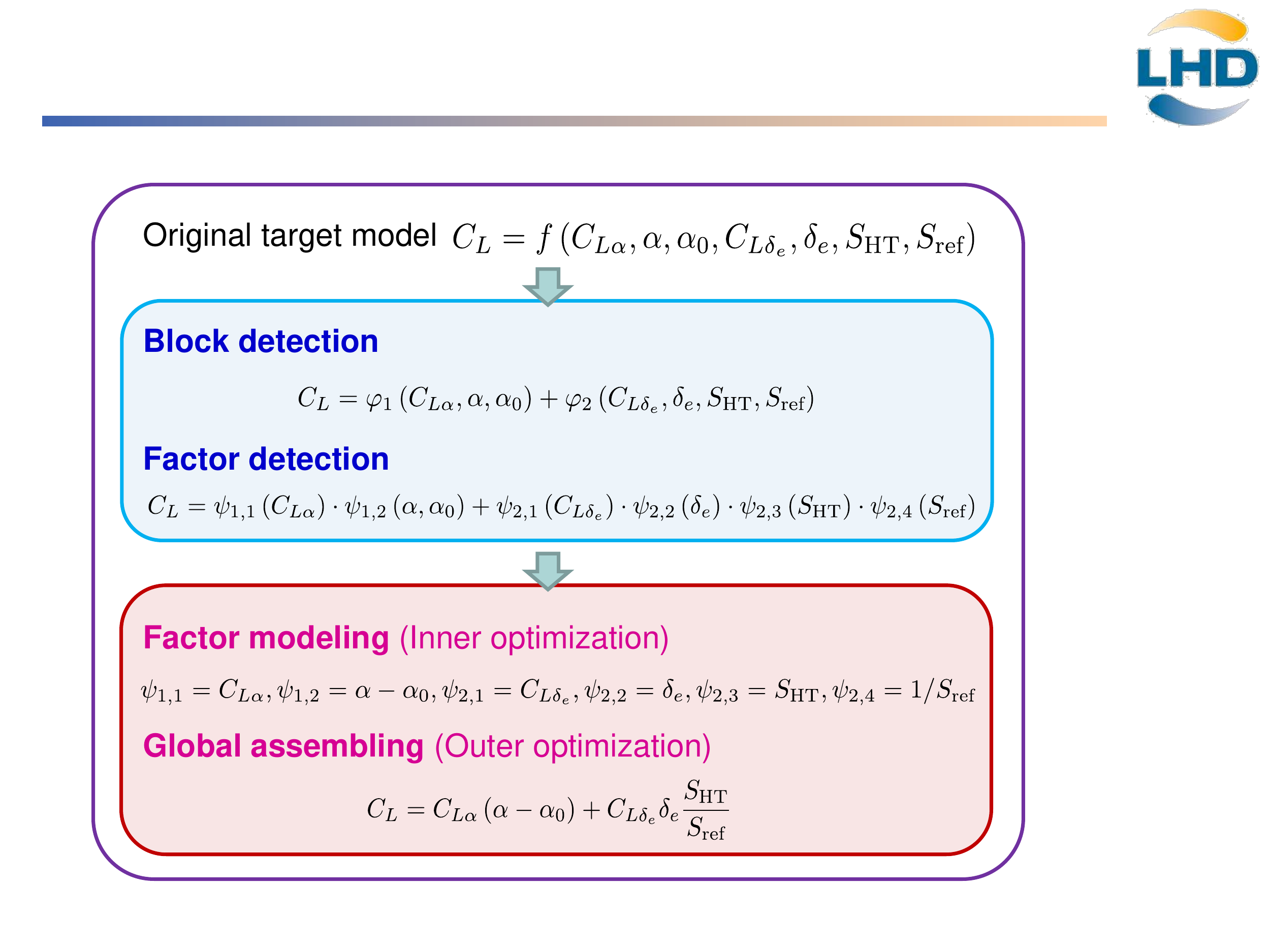}
\caption{An example procedure of BBP: The modeling of the Eq. (\ref{CL}).}
\label{fig1}
\end{figure}

\subsection{Optimization engine}
\label{Section3.4}
Factor modeling could be easily realized by global optimization algorithms or genetic programming. However, a few differences between these two methods should be considered.

In BBP, when using a global optimization method, the structure of the factors (function models) should be pre-established. For instance, functions that involve uni-variable and bi-variables should be set for function models. Sequence search and optimization methods is suitable for global optimization strategies. This means a certain function model will be determined, provided that the fitting error is small enough (e.g. mean square error $\leqslant {10^{ - 8}}$).

In genetic programming, arithmetic operations (e.g. $+$, $-$, $\times$, $\div$) and mathematical functions (e.g. $\sin$, $\cos$, $\exp$, $\ln$) should be pre-established instead of function models. The search process of GP is stochastic. This makes it easily reach premature convergence. In Section \ref{Section5}, LDSE and GPTIPS, are both chosen to test the performance of BBP.

The choice of LDSE and GPTIPS as the optimization engine is based on the obvious fact that the factor modeling might contain the non-convex optimization process (e.g. $\psi  = {m_1}\sin \left( {{m_2}x + {m_3}} \right)$). LDSE and GPTIPS can make our program easy to use, but may have an additional calculation cost. To further speed up the optimization process, the convex optimization algorithm \cite{Wei2015-1,Wei2015-2,Wei2017} is valid provided that the factor modeling could be reduced to a linear programming, quadratic programming, and semi-definite programming problem (e.g. $\psi  = {m_1}{x^{{m_2}}} + {m_3}$).

\subsection{Remarks for engineering applications}
\label{Section3.5}
The proposed method is described with functions of explicit expressions. For practical applications, data-driven modeling problems are more common, which means the explicit expression might be unavailable. To use the proposed BBP method, one can construct a surrogate model of the black-box type to represent the underlying target function \cite{Forrester2009}, and the surrogate model could be used as the target model in the BBP method. The other modeling steps are the same as described above.

Data noise is another issue to consider for applying BBP. In practical applications, the input-response data usually involves noises such as measure errors and/or system noises. Note that during the above surrogate modeling (e.g. Kriging regression \cite{Lophaven2002}) process, some noises could be suppressed. Meanwhile, the exact BiCT is now unnecessary, since the surrogate model is also not the true target function. We can use a special technique called Soft BiCT, where $\left| \rho  \right|$ is set to $\left| \rho  \right| = 1 - \varepsilon$ ($\varepsilon$ is a small positive number) instead of $\left| \rho  \right| = 1$ in BiCT. Multiple BiCT could also further suppress the noises, where each variable is fixed to more pairs of vectors (one pair in BiCT). Detailed discussions will be provided in future studies.

\section{Numerical results and discussion}
\label{Section4}
The proposed BBP is implemented in Matlab/Octave. In order to test the performance of BBP, two different optimization engines, LDSE \cite{Luo2012-LDSE} and GPTIPS \cite{Searson2010}, are used. For ease of use, a Boolean variable is used on the two selected methods. Numerical experiments on 10 cases of completely separable or partially separable target functions, as given in \ref{appendixB}, are conducted. These cases help evaluate BBP's overall capability of structure and coefficient optimization. Computational efficiency is analyzed in Section \ref{Section4.3}.

\subsection{LDSE-powered BBP}
\label{Section4.1}
We choose a global optimization algorithm, LDSE \cite{Luo2012-LDSE}, as our optimization engine. LDSE is a hybrid evolutionary algorithm for continuous global optimization. In Table \ref{LDSE-table}, case number, dimension, domain, and number of sampling points are denoted as No, Dim, Domain, and No.samples, respectively. Additionally, we record the time $T_d$ (the block and factor detection), and $T_{\text{BBP}}$ (the entire computation time of BBP), to test the efficiency of block and factor detection.

\subsubsection{Control parameter setting}
\label{Section4.1.1}
The calculation conditions are shown in Table \ref{LDSE-table}. The number of sampling points for each independent variable is 100. The regions for cases 1--5 and 7--10 are chosen as $[-3,3]^3$,$[-3,3]^4$, $[-3,3]^5$, and $[-3,3]^6$ for three-dimensional (3D), 4D, 5D, and 6D problems, respectively, while case 7 is $[1,3]^5$. The control parameters in LDSE are set as follows. The upper and lower bounds of fitting parameters are set as $-50$ and $50$. The population size $N_p$ is set to $N_p=10+10d$, where $d$ is the dimension of the problem. The maximum generations is set to $3N_p$. Note that the maximum number of partially separable variables in all target models is two in our tests. Hence, our uni-variable and bi-variables function library of BBP could be set as in Table \ref{models-table}. Recalling from Section \ref{Section3.3}, sequence search and optimization is used in BBP. The search will exit immediately if the mean square error is small enough (MSE $\leqslant {\varepsilon _{{\text{target}}}}$), and the tolerance (fitting error) is ${\varepsilon _{{\text{target}}}} = {10^{ - 6}}$. In order to reduce the effect of randomness, each test case is executed 10 times.

\begin{table}[h]
\centering
\caption{Uni-variable and bi-variables preseted models.}
\label{models-table}
\begin{tabular}{lll}
\hline\hline
No. & Uni-variable model                                & Bi-variables model                                                                                                                                                                                                                             \\ \hline
1   & $k\left( {{x^{{m_1}}} + {m_2}} \right)$           & $k\left( {{m_1}{x_1} + {m_2}{x_2} + {m_3}} \right)$                                                                                                                                                                                            \\
2   & $k\left( {{e^{{m_1}x}} + {m_2}} \right)$          & $k\left[ {{{\left( {{m_1}{x_1} + {m_2}} \right)} \mathord{\left/ {\vphantom {{\left( {{m_1}{x_1} + {m_2}} \right)} {\left( {{m_3}{x_2} + {m_4}} \right)}}} \right. \kern-\nulldelimiterspace} {\left( {{m_3}{x_2} + {m_4}} \right)}}} \right]$ \\
3   & $k\sin \left( {{m_1}{x^{{m_2}}} + {m_3}} \right)$ & $k\left( {{e^{{m_1}{x_1}{x_2}}} + {m_2}} \right)$                                                                                                                                                                                              \\
4   & $k\log \left( {{m_1}x + {m_2}} \right)$           & $k\sin \left( {{m_1}{x_1} + {m_2}{x_2} + {m_3}{x_1}{x_2} + {m_4}} \right)$                                                                                                                                                                     \\ \hline\hline
\end{tabular}
\end{table}

\subsubsection{Numerical results and discussion}
\label{Section4.1.2}
Numerical results show that LDSE-powered BBP successfully recovered all target functions exactly in sense of double precision. Once the uni- and bi-variables models are pre-set, the sequence search method makes BBP easy to find the best regression model. In practical applications, more function models could be added to the function library of BBP, provided that they are needed. On the other hand, as sketched in Table \ref{LDSE-table}, the calculation time of the separability detection $T_d$ is almost negligible. This test group shows that BBP has a good capability of structure and coefficient optimization for highly nonlinear system.
\begin{table}[h]
\centering
\caption{Performance of LDSE-powered BBP.}
\label{LDSE-table}
\begin{tabular}{lllll}
\hline\hline
Case No. & Dim & Domain                        & No. samples & $T_d/T_{\text{BBP}}$(\%) \\ \hline
1        & 3   & ${\left[ { - 3,3} \right]^3}$ & 300         & 4.35                     \\
2        & 3   & ${\left[ { - 3,3} \right]^3}$ & 300         & 2.38                     \\
3        & 4   & ${\left[ { - 3,3} \right]^4}$ & 400         & 4.7                      \\
4        & 4   & ${\left[ { - 3,3} \right]^4}$ & 400         & 2.51                     \\
5        & 4   & ${\left[ { - 3,3} \right]^4}$ & 400         & 3.32                     \\
6        & 5   & ${\left[ { 1,4} \right]^5}$   & 300         & 1.98                     \\
7        & 5   & ${\left[ { - 3,3} \right]^5}$ & 500         & 4.42                     \\
8        & 5   & ${\left[ { - 3,3} \right]^5}$ & 500         & 1.68                     \\
9        & 6   & ${\left[ { - 3,3} \right]^6}$ & 600         & 2.86                     \\
10       & 6   & ${\left[ { - 3,3} \right]^6}$ & 600         & 3.38                     \\ \hline\hline
\end{tabular}
\end{table}

\subsection{GPTIPS-powered BBP}
\label{Section4.2}
We choose a kind of genetic programming technique, GPTIPS \cite{Searson2010}, as the optimization engine. GPTIPS is a Matlab toolbox based on multi-gene genetic programming. It has been widely used in many research studies \cite{Garg2014,Alavi2017,Kaydani2014}. To provide BBP an overall evaluation of its performance for acceleration, the acceleration rate $\eta$ is defined as
\begin{equation}
\label{speedup}
\eta  = \frac{{{T_{{\text{GPTIPS}}}}}}{{{T_{{\text{BBP}}}}}},
\end{equation}
where $T_{\text{GPTIPS}}$ is the computation time of GPTIPS, and $T_{\text{BBP}}$ is the computation time of BBP driven by GPTIPS. The full names of the notations in Table \ref{LDSE-table} are the case number (Case No.), the range of mean square error of the regression model for all runs (MSE), the average computation time for all runs ($T$), and remarks of BBP test.

\subsubsection{Control parameter setting}
\label{Section4.2.1}
Similar to Section \ref{Section4.1.1}, the target models, search regions, and the number of sampling points are the same as the aforementioned test group. The control parameters of GPTIPS are set as follows. The population size $N_p=100$ and the maximum generations for re-initialization $T$ are 100,000. To reduce the influence of randomness, 20 runs are completed for each case. The termination condition is MSE $\leqslant {\varepsilon _{{\text{target}}}}$, ${\varepsilon _{{\text{target}}}} = {10^{ - 8}}$. In other words, the optimization of each factor will terminate immediately if the regression model (or its equivalent alternative) is detected, and restart automatically if it fails until generation $T$. The multi-gene of GPTIPS is switched off.

\subsubsection{Numerical results and discussion}
\label{Section4.2.2}
Table \ref{GPTIPS} shows the average performance of the 20 independent runs with different initial populations. In this test group, using the given control parameters, GPTIPS failed to obtain the exact regression model or the approximate model with the default accuracy (MSE $\leqslant 10^{-8}$) in almost 20 runs except case 2. This situation becomes even worse with increasing problem size (dimension of the problem). Additionally, as the result of ${T_{{\text{GPTIPS}}}}/{T_{{\text{BBP}}}}$ shown in Table \ref{GPTIPS}, the computational efficiency increases several orders of magnitude. This is because the computation time of BBP is determined by the dimension and complexity of each factor, not by the entire target model. This explains why BBP converges much faster than the original GPTIPS. Good performance for acceleration, structure optimization, and coefficient optimization show the potential of BBP to be applied in practical applications.

\begin{table}[h]
\scriptsize
\centering
\caption{Performance of GPTIPS and GPTIPS-powered BBP.}
\label{GPTIPS}
\begin{tabular}{lll|llll}
\hline\hline
\multirow{2}{*}{\begin{tabular}[c]{@{}l@{}}Case\\ No.\end{tabular}} & \multicolumn{2}{l|}{Results of GPTIPS}                                      & \multicolumn{4}{l}{Results of GPTIPS-powered BBP}                                                                                                                  \\ \cline{2-7} 
                                                                    & ${\text{MSE}}_{\text{GPTIPS}}$                   & $T_{\text{GPTIPS}}(s)$   & ${\text{MSE}}_{\text{BBP}}$                     & $T_{\text{BBP}}(s)$        & $\eta = T_{\text{GPTIPS}}/ T_{\text{BBP}}$ & Remarks of BBP          \\ \hline
1                                                                   & $\left[ {5.1,9.2} \right] \times {10^{ - 1}}$    & $\gg 6.23 \times {10^3}$ & $\leqslant {\varepsilon _{{\text{target}}}}$    & 1512.6                     & $>4.12$                                    & 5 runs failed           \\
2                                                                   & $\leqslant {\varepsilon _{{\text{target}}}}$     & 323.86                   & $\leqslant {\varepsilon _{{\text{target}}}}$    & 3.94                       & 82.2                                       & Solutions are all exact \\
3                                                                   & $\left[ {1.5,23.6} \right] \times {10^{ - 2}}$   & $\gg 7.41 \times {10^3}$ & $\leqslant {\varepsilon _{{\text{target}}}}$    & 3643.3                     & $>2.03$                                    & 11 runs failed          \\
4                                                                   & $\left[ {8.1,16.2} \right] \times {10^{ - 2}}$   & $\gg 6.16 \times {10^3}$ & $\leqslant {\varepsilon _{{\text{target}}}}$    & 903.87                     & $>6.8$                                     & 4 runs failed           \\
5                                                                   & $\left[ {4.5,7.3} \right] \times {10^{ - 1}}$    & $\gg 6.68 \times{10^3}$  & $\leqslant {\varepsilon _{{\text{target}}}}$    & 26.65                      & $>250.65$                                  & Solutions are all exact \\
6                                                                   & $\left[ {2.31,9.6} \right] \times {10^{ - 2}}$   & $\gg 6.31 \times {10^3}$ & $\leqslant {\varepsilon _{{\text{target}}}}$    & 4416.07                    & $>1.67$                                    & 7 runs failed           \\
7                                                                   & $\left[ {1.22,3.96} \right] \times {10^{ - 1}}$  & $\gg 8.52 \times {10^3}$ & $\leqslant {\varepsilon _{{\text{target}}}}$    & 11.57                      & $>736.39$                                  & Solutions are all exact \\
8                                                                   & $\left[ {2.1,37.2} \right] \times {10^{ - 1}}$   & $\gg 7.13 \times {10^3}$ & $\left[ {9.16,32.3} \right] \times {10^{ - 2}}$ & $\gg 1.3721 \times {10^4}$ & None                                       & All runs failed         \\
9                                                                   & $\left[ {5.4,56.3} \right] \times {10^{ - 2}}$   & $\gg 6.24 \times {10^3}$ & $\left[ {1.68,12.9} \right] \times {10^{ - 2}}$ & $\gg 6.63 \times {10^3}$   & None                                       & All runs failed         \\
10                                                                  & $\left[ {5.86,99.16} \right] \times {10^{ - 1}}$ & $\gg 7.36 \times {10^3}$ & $\leqslant {\varepsilon _{{\text{target}}}}$    & 11.62                      & $>708.26$                                  & Solutions are all exact \\ \hline\hline
\end{tabular}
\end{table}

\subsection{Computational efficiency analysis}
\label{Section4.3}
We compare the target functional spaces of conventional GP method (e.g. GPTIPS) and GPTIPS-powered BBP. The search space of BBP is each factor of the target model, not the entire target model. It is obvious that GPTIPS has a larger target function space, and GPTIPS-powered BBP might be considered a special case of GPTIPS. This is the reason why GPTIPS-powered BBP is more effective and efficient than conventional GPTIPS.

The computing time ($t$) of BBP consists of three parts:
\begin{equation}
t=t_1+t_2+t_3
\label{eqTiming}
\end{equation}
where $t_1$ is for the separability detection, $t_2$ for factors modeling, and $t_3$ for global assembling.  In \cite{Luo2017}, authors have demonstrated that both the separability detection and function recover processes are double-precision operations and thus cost much less time than the factor determination process. $t_2$ is the most expensive part used to construct a data-driven model. That is, $t \approx t_2$.

As shown in Table \ref{GPTIPS}, the CPU time for determining all factors ($t_2$) is much less than that of the target function directly ($t_d$). Therefore, in practical applications, we do not consider the computation complexity of a separable target model, but each factor of it.

\section{Conclusion}
\label{Section5}
We established a more general separable model with mixed binary operators. In order to obtain the structure of the generalized model, a block building programming (BBP) method is proposed for symbolic regression. BBP reveals the target separable model by a block and factor detection process, which divides the original model into a number of blocks, and further into factors. The method can be considered as a bi-level divide-and-conquer (D\&C) method. The separability is detected by a generalized BiCT. The factors could be easily determined by an existing optimization engine (e.g. genetic programming). Thus BBP can reduce the complexity of the optimization model, and make large reductions to the original search space. Two different optimization engines, LDSE and GPTIPS, have been applied to test the performance of BBP on 10 symbolic regression problems. Numerical results show that BBP has a good capability of structure and coefficient optimization with high computational efficiency. These advantages make BBP a potential method for modeling complex nonlinear systems in various research fields.

For future work, we plan to generalize the mathematical form of the separable function. In Definition \ref{def1}, all variables appear only once in the separable function. However, in practical applications, this condition is still too strong and is not easy to satisfy. In fact, many models have quasi-separable features. For example, the flow pass of a circular cylinder is a classical problem in fluid dynamics \cite{Anderson2011}. A valid stream function for the inviscid, incompressible flow pass of a circular cylinder of radius $R$ is
\begin{equation}
\label{flowequ}
\psi = \left( {{V_\infty }r\sin \theta } \right)\left( {1 - \frac{{{R^2}}}{{{r^2}}}} \right) + \frac{\Gamma }{{2\pi }}\ln \frac{r}{R},
\end{equation}
Eq. (\ref{flowequ}) is expressed in terms of polar coordinates, where $x=r\cos\theta$ and $y=r\sin\theta$. $V_\infty$ is the freestream velocity. $\Gamma$ is the strength of vortex flow. $R$ is the radius of the cylinder. Eq. (\ref{flowequ}) can be considered quasi-separable. That is, some variables (e.g. variable ${V_\infty }$, $\theta$, and $\Gamma$ of Eq. (\ref{flowequ})) appear only once in a concerned target model, while other variables (e.g. variable $r$ and $R$ of Eq. (\ref{flowequ})) appear more than once. This makes Eq. (\ref{flowequ}) inconsistent with the definition of the separable function. Such complicated model structures would be analyzed in further studies.

\section*{Acknowledgements}
This work was supported by the National Natural Science Foundation of China (Grant No. 11532014).

\appendix
\section{Proof of Theorem \ref{theorem3}}
\label{appendixA}
\begin{proof}
To prove the sufficient condition, consider the binary operator ${ \otimes _1}$. Eq. (\ref{SeparFuncEqu}) could be written as
\begin{equation}
\label{Appendix-1}
f\left( {\mathbf{x}} \right) = {\alpha _0} + {\alpha _1}{f_1}\left( {{{\mathbf{I}}^{\left( 1 \right)}}{\mathbf{x}}} \right){ \otimes _2}{\alpha _2}{f_2}\left( {{{\mathbf{I}}^{\left( 2 \right)}}{\mathbf{x}}} \right){ \otimes _3} \cdots { \otimes _m}{\alpha _m}{f_m}\left( {{{\mathbf{I}}^{\left( m \right)}}{\mathbf{x}}} \right),
\end{equation}
or
\begin{equation}
\label{Appendix-2}
f\left( {\mathbf{x}} \right) = {\alpha _0} \times {\alpha _1}{f_1}\left( {{{\mathbf{I}}^{\left( 1 \right)}}{\mathbf{x}}} \right){ \otimes _2}{\alpha _2}{f_2}\left( {{{\mathbf{I}}^{\left( 2 \right)}}{\mathbf{x}}} \right){ \otimes _3} \cdots { \otimes _m}{\alpha _m}{f_m}\left( {{{\mathbf{I}}^{\left( m \right)}}{\mathbf{x}}} \right).
\end{equation}

Consider the position of each binary plus operator ($+$). Assume that there are $p-1$ ($p\geqslant 1$) binary plus operators ($+$) in Eq. (\ref{SeparFuncEqu}). If $p>1$, assume the first binary operator plus ($+$) (except ${ \otimes _1}$) of Eq. (\ref{SeparFuncEqu}) appears in the middle of the sub-functions of variables $x_1, x_2,\cdots,x_{s_1}$ (including a number of $q_1+q_2$ sub-functions) and $x_{{s_1}+1}, x_{{s_1}+2},\cdots,x_{n}$, that is
\begin{equation}
\label{Appendix-5}
f\left( {\mathbf{x}} \right) = {\alpha _0} + {{\tilde \beta }_1}{\varphi _1}\left( {{x_1},{x_2}, \cdots {x_{{s_1}}}} \right) + {{\tilde \beta }_2}{{\tilde \varphi }_2}\left( {{x_{{s_1} + 1}},{x_{{s_1} + 2}}, \cdots {x_n}} \right),
\end{equation}
or
\begin{equation}
\label{Appendix-6}
f\left( {\mathbf{x}} \right) = {\alpha _0} \times {{\tilde \beta }_1}{\varphi _1}\left( {{x_1},{x_2}, \cdots {x_{{s_1}}}} \right) + {{\tilde \beta }_2}{{\tilde \varphi }_2}\left( {{x_{{s_1} + 1}},{x_{{s_1} + 2}}, \cdots {x_n}} \right).
\end{equation}

Assume that the second binary plus operator ($+$) appears in the middle of the sub-functions of variables $x_1, x_2,\cdots,x_{{s_1}+{s_2}}$ (included a number of $q_1+q_2$ sub-functions) and $x_{{s_1}+{s_2}+1}, x_{{s_1}+{s_2}+2},\cdots,x_{n}$, that is
\begin{equation}
\label{Appendix-7}
f\left( {\mathbf{x}} \right) = {\alpha _0} + {{\tilde \beta }_1}{\varphi _1}\left( {{x_1}, \cdots ,{x_{{s_1}}}} \right) + {{\tilde \beta }_2}{\varphi _2}\left( {{x_{{s_1} + 1}}, \cdots ,{x_{{s_1} + {s_2}}}} \right) + {{\tilde \beta }_3}{{\tilde \varphi }_3}\left( {{x_{{s_1} + {s_2} + 1}}, \cdots ,{x_n}} \right),
\end{equation}
or
\begin{equation}
\label{Appendix-8}
f\left( {\mathbf{x}} \right) = {\alpha _0} \times {{\tilde \beta }_1}{\varphi _1}\left( {{x_1}, \cdots ,{x_{{s_1}}}} \right) + {{\tilde \beta }_2}{\varphi _2}\left( {{x_{{s_1} + 1}}, \cdots ,{x_{{s_1} + {s_2}}}} \right) + {{\tilde \beta }_3}{{\tilde \varphi }_3}\left( {{x_{{s_1} + {s_2} + 1}}, \cdots ,{x_n}} \right).
\end{equation}

The position of the rest $p-4$ binary plus operators ($+$) could be determined in the same way. We obtain
\begin{equation}
\label{Appendix-9}
f\left( {\mathbf{x}} \right) = {\alpha _0} + \sum\limits_{i = 1}^p {{{\tilde \beta }_i}{\varphi _i}\left( {{{\mathbf{I}}^{\left( i \right)}}{\mathbf{x}}} \right)},
\end{equation}
or
\begin{equation}
\label{Appendix-10}
f\left( {\mathbf{x}} \right) = {\alpha _0} \times \sum\limits_{i = 1}^p {{{\tilde \beta }_i}{\varphi _i}\left( {{{\mathbf{I}}^{\left( i \right)}}{\mathbf{x}}} \right)}.
\end{equation}
${{\mathbf{I}}^{\left( i \right)}} \in {\mathbb{R}^{{s_i} \times n}}$ is the partitioned matrix of the identity matrix ${\mathbf{I}} \in {\mathbb{R}^{n \times n}}$, $\sum\limits_{i = 1}^p {{s_i}}  = n$. For Eq. (\ref{Appendix-9}), ${\beta _0}={\alpha _0}$, ${\beta _i}={{\tilde \beta }_i}$. For Eq. (\ref{Appendix-10}), ${\beta _0} = 0$, ${\beta _i}={\alpha _0}{{\tilde \beta }_i}$. Hence, the left-hand side of Eq. (\ref{SeparFuncEquDetail}) can be obtained.

If $p=1$, then there is no binary plus operator ($+$) (except ${ \otimes _1}$) in Eq. (\ref{SeparFuncEqu}). Under this condition, it is obvious that Eq. (\ref{Appendix-9}) and (\ref{Appendix-10}) are still satisfied.

Now decide the position of each binary times operator ($\times$). Consider the first sub-function ${{\varphi _1}\left( {{{\mathbf{I}}^{\left( 1 \right)}}{\mathbf{x}}} \right)}$ in Eq. (\ref{Appendix-9}) and (\ref{Appendix-10}). Assume that the first binary times operator ($\times$) of the sub-function ${{\varphi _1}\left( {{{\mathbf{I}}^{\left( 1 \right)}}{\mathbf{x}}} \right)}$ appears in the middle of the sub-functions of variables $x_1, x_2,\cdots,x_{s_{1,1}}$ and $x_{s_{{1,1}}+1}, x_{s_{{1,1}}+2},\cdots,x_{s_1}$, that is
\begin{equation}
\label{times1}
{\varphi _1}\left( {{{\mathbf{I}}^{\left( 1 \right)}}{\mathbf{x}}} \right)={\psi _{1,1}}\left( {{x_1},{x_2}, \cdots ,{x_{{s_{1,1}}}}} \right) \cdot {{\tilde \psi }_1}\left( {{x_{{s_{1,1}} + 1}},{x_{{s_{1,1}} + 2}}, \cdots ,{x_{{s_1}}}} \right).
\end{equation}
The second binary times operator ($\times$) of the sub-function ${{\varphi _1}\left( {{{\mathbf{I}}^{\left( 1 \right)}}{\mathbf{x}}} \right)}$ appears in the middle of the sub-functions of variables ${x_1},{x_2}, \cdots ,{x_{{s_{1,1}} + {s_{1,2}}}}$ and ${x_{{s_{1,1}} + {s_{1,2}} + 1}},{x_{{s_{1,1}} + {s_{1,2}} + 2}}, \cdots ,{x_{{s_1}}}$, that is
\begin{equation}
\label{times2}
{\varphi _1}\left( {{{\mathbf{I}}^{\left( 1 \right)}}{\mathbf{x}}} \right)={\psi _{1,1}}\left( {{x_1}, \cdots ,{x_{{s_{1,1}}}}} \right) \cdot {\psi _{1,2}}\left( {{x_{{s_{1,1}} + 1}}, \cdots ,{x_{{s_{1,1}} + {s_{1,2}}}}} \right) \cdot {{\tilde \psi }_2}\left( {{x_{{s_{1,1}} + {s_{1,2}} + 1}}, \cdots ,{x_{{s_1}}}} \right).
\end{equation}

The position of the rest binary times operators ($\times$) of the sub-function ${{\varphi _1}\left( {{{\mathbf{I}}^{\left( 1 \right)}}{\mathbf{x}}} \right)}$ could be determined in the same way. Then, ${\varphi _1}\left( {{{\mathbf{I}}^{\left( 1 \right)}}{\mathbf{x}}} \right)$ could be rewritten as 
\begin{equation}
{\varphi _1}\left( {{{\mathbf{I}}^{\left( 1 \right)}}{\mathbf{x}}} \right) = \prod\limits_{j = 1}^{{q_1}} {{\psi _{1,j}}\left( {{\mathbf{I}}_j^{\left( 1 \right)}{\mathbf{x}}} \right)},
\end{equation}
where ${{\mathbf{I}}_j^{\left( 1 \right)}} \in {\mathbb{R}^{{s_{1,j}} \times n}}$ is the partitioned matrix of ${{\mathbf{I}}^{\left( 1 \right)}}$, namely ${{\mathbf{I}}^{\left( 1 \right)}} = \left[ {\begin{array}{*{20}{c}}
  {{\mathbf{I}}_1^{\left( 1 \right)}}&{{\mathbf{I}}_2^{\left( 1 \right)}}& \cdots &{{\mathbf{I}}_{{q_1}}^{\left( 1 \right)}} 
\end{array}} \right]$, $\sum\limits_{j = 1}^{{q_1}} {{s_{1,j}}}  = {s_1}$. Hence, for arbitrary sub-function ${\varphi _i}\left( {{{\mathbf{I}}^{\left( i \right)}}{\mathbf{x}}} \right)$, we have
\begin{equation}
\label{times3}
{\varphi _i}\left( {{{\mathbf{I}}^{\left( i \right)}}{\mathbf{x}}} \right) = \prod\limits_{j = 1}^{{q_i}} {{\psi _{i,j}}\left( {{\mathbf{I}}_j^{\left( i \right)}{\mathbf{x}}} \right)},
\end{equation}
where ${{{\mathbf{I}}_{j}^{\left( i \right)}}\in {\mathbb{R}^{{s_{i,j}} \times n}}} $ is the partitioned matrix of ${{\mathbf{I}}^{\left( i \right)}}$, namely ${{\mathbf{I}}^{\left( i \right)}} = \left[ {\begin{array}{*{20}{c}}
  {{\mathbf{I}}_1^{\left( i \right)}}&{{\mathbf{I}}_2^{\left( i \right)}}& \cdots &{{\mathbf{I}}_{{q_i}}^{\left( i \right)}} 
\end{array}} \right]$, $\sum\limits_{j = 1}^{{q_i}} {{s_{i,j}} = {s_i}}$, $\sum\limits_{i = 1}^p {{q_i}}  = m$. Substituting Eq. (\ref{times3}) into Eq. (\ref{Appendix-9}) and (\ref{Appendix-10}), the right-hand side of Eq. (\ref{SeparFuncEquDetail}) can be obtained.

To prove the necessary condition, expand Eq. (\ref{SeparFuncEquDetail}). Since $\sum\limits_{i = 1}^p {{q_i}}  = m$, there are $m$ sub-functions $f_i$ and each function can be connected with a binary operator $\otimes_i$. Then, the Eq. (\ref{SeparFuncEqu}) can be easily obtained.

\end{proof}

\section{10 target models of numerical experiments}
\label{appendixB}
The target models which are tested in Section \ref{Section4} with all the  blocks boxed are given as follows:
\begin{description}
\item Case 1. $f\left( {\mathbf{x}} \right) = 1.2{\text{  +  }}10 * \boxed{\sin \left( {2{x_1} - {x_3}} \right)} - 3 * \boxed{x_2^2}$, where ${x_i} \in \left[ { - 3,3} \right],i = 1,2,3.$
\item Case 2. $f\left( {\mathbf{x}} \right) = 0.5 * \boxed{{e^{{x_3}}} * \sin {x_1} * \cos {x_2}}$, where ${x_i} \in \left[ { - 3,3} \right],i = 1,2,3.$
\item Case 3. $f\left( {\mathbf{x}} \right) = \boxed{\cos \left( {{x_1} + {x_2}} \right)} + \boxed{\sin \left( {3{x_3} - {x_4}} \right)}$, where ${x_i} \in \left[ { - 3,3} \right],i = 1,2,3,4.$
\item Case 4. $f\left( {\mathbf{x}} \right) = 5 * \boxed{\frac{{\sin \left( {3{x_1}{x_2}} \right)}}{{{x_3} + {x_4}}}}$, where ${x_i} \in \left[ { - 3,3} \right],i = 1,2,3,4.$
\item Case 5. $f\left( {\mathbf{x}} \right) = 2 * \boxed{{x_1} * \sin \left( {{x_2} + {x_3}} \right)} - \boxed{\cos {x_4}}$, where ${x_i} \in \left[ { - 3,3} \right],i = 1,2,3,4.$
\item Case 6. $f\left( {\mathbf{x}} \right) = 10 + 0.2 * \boxed{{x_1}} - 5 * \boxed{\sin \left( {5{x_2} + {x_3}} \right)} + \boxed{\ln \left( {3{x_4} + 1.2} \right)} - 1.2 * \boxed{{e^{0.5{x_5}}}}$, where ${x_i} \in \left[ {1,4} \right],i = 1,2, \cdots ,5.$
\item Case 7. $f\left( {\mathbf{x}} \right) = 10 * \boxed{\frac{{\sin \left( {{x_1}{x_2}} \right) * {x_3}}}{{{x_4} + {x_5}}}}$, where ${x_i} \in \left[ { - 3,3} \right],i = 1,2, \cdots ,5.$
\item Case 8. $f\left( {\mathbf{x}} \right) = 1.2{\text{  +  }}2 * \boxed{{x_4} * \cos {x_2}} + 0.5 * \boxed{{e^{1.2{x_3}}} * \sin 3{x_1}} - 2 * \boxed{\cos \left( {1.5{x_5} + 5} \right)}$, where ${x_i} \in \left[ { - 3,3} \right],i = 1,2, \cdots ,5.$
\item Case 9. $f\left( {\mathbf{x}} \right) = 100 * \boxed{\frac{{\cos \left( {{x_3}{x_4}} \right)}}{{{e^{{x_1}}} * {x_2}^{1.2}}} * \sin \left( {1.5{x_5} - 2{x_6}} \right)}$, where ${x_i} \in \left[ { - 3,3} \right],i = 1,2, \cdots ,6.$
\item Case 10. $f\left( {\mathbf{x}} \right) = \boxed{\frac{{{x_1} + {x_2}}}{{{x_3}}}} + \boxed{{x_4} * \sin \left( {{x_5}{x_6}} \right)}$, where ${x_i} \in \left[ { - 3,3} \right],i = 1,2, \cdots ,6.$
\end{description}

\bibliographystyle{elsarticle-num}

\bibliography{BBP}

\end{document}